\documentclass[10pt,conference]{IEEEtran}
\IEEEoverridecommandlockouts

\def\BibTeX{{\rm B\kern-.05em{\sc i\kern-.025em b}\kern-.08em
    T\kern-.1667em\lower.7ex\hbox{E}\kern-.125emX}}

\usepackage{multirow}
\usepackage{graphicx}
\usepackage{array}
\usepackage{xcolor}
\usepackage{amsmath,amssymb,amsfonts}
\DeclareMathAlphabet{\mathcal}{OMS}{cmsy}{m}{n}
\usepackage[mathscr]{eucal}
\DeclareMathAlphabet\mathbfcal{OMS}{cmsy}{b}{n}
\usepackage{multirow}
\usepackage[ruled,vlined,linesnumbered]{algorithm2e}
\usepackage{amsmath}
\usepackage{hyperref}

\newtheorem{theorem}{\textbf{Theorem}}
\newtheorem{definition}{\textbf{Definition}}
\newtheorem{proof}{Proof}

\usepackage[font={small}]{caption}
\usepackage[font={footnotesize}]{subcaption}
\usepackage{cancel}
\usepackage{epstopdf}
\let\oldnl\nl
\newcommand{\nonl}{\renewcommand{\nl}{\let\nl\oldnl}}

\usepackage{graphicx}
\usepackage{textcomp}
\usepackage{xcolor}
\usepackage[normalem]{ulem}
\usepackage{cite}

\begin{document}
\title{
\begin{large}
 IEEE Global Communications Conference (GLOBECOM) 2022\\
\end{large}
Securing Federated Learning against Overwhelming Collusive Attackers}
\author{\IEEEauthorblockN{
Priyesh Ranjan,
	Ashish Gupta,
	Federico Corò,
 	and Sajal K.~Das}
  \IEEEauthorblockA{
  Department of Computer Science, Missouri University of Science and Technology, Rolla, USA \\ 
 {\{pr8pf, ashish.gupta, federico.coro, sdas\}@mst.edu}\\}} 
\maketitle

\begin{abstract}
In the era of a data-driven society with the ubiquity of Internet of Things (IoT) devices storing large amounts of data localized at different places, distributed learning has gained a lot of traction, however, assuming independent and identically distributed data (iid) across the devices. While relaxing this assumption that anyway does not hold in reality due to the heterogeneous nature of devices, federated learning (FL) has emerged as a privacy-preserving solution to train a collaborative model over non-iid data distributed across a massive number of devices. However, the appearance of malicious devices (attackers), who intend to corrupt the FL model, is inevitable due to unrestricted participation. In this work, we aim to identify such attackers and mitigate their impact on the model, essentially under a setting of bidirectional label flipping attacks with collusion. We propose two graph theoretic algorithms, based on Minimum Spanning Tree and $k$-Densest graph, by leveraging correlations between local models. Our FL model can nullify the influence of attackers even when they are up to 70\% of all the clients whereas prior works could not afford more than 50\% of clients as attackers.  The effectiveness of our algorithms is ascertained through experiments on two benchmark datasets, namely MNIST and Fashion-MNIST, with overwhelming attackers. We establish the superiority of our algorithms over the existing ones using accuracy, attack success rate, and early detection round.
\end{abstract}
\begin{keywords}
Attackers, federated learning, label flipping
\end{keywords}

\section{Introduction}
The proliferation of smartphones and IoT devices with significant computing capabilities has led to a steep growth in the adoption of machine learning techniques in our daily routine. 
These devices generate a large amount of data, traditionally processed at a remote server, causing a waste of bandwidth and exposing the privacy of the users as the data may include sensitive information. To address these issues, Google researchers came up with a distributed learning paradigm, called Federated Learning (FL)~\cite{mcmahan2017communication}, in which multiple devices (or clients) can collaborate to produce an accurate and generalized model, usually in the presence of a remote server while keeping the data private.  
The concept involves training a model using the individual data fragments of local devices, followed by passing that model (essentially weights/parameters) to the server for aggregation. Thereafter, the updated model, also referred to as the global model, is relayed back to the clients which marks the end of a single round of the FL process. 

However, the lack of transparency invites adversaries who may pose as participants with the intention of corrupting the process by supplying poisoned local models to the server. It not only reduces the performance of the global model but also influences its convergence. 
Such an adversarial attack can have multi-fold objectives that include promoting the outcome of a particular class by flipping the labels of the corresponding class~\cite{259745,steinhardt2017certified}, poisoning the data by adding backdoor elements~\cite{xie2021crfl} in their local data shards, or providing random parameters (i.e., model replacement) thus diverging the model from optimal solution~\cite{blanchard2017machine,bagdasaryan2020backdoor}. Figure~\ref{fig:char} illustrates an example scenario for the digit classification task using images of handwritten digits where two of the participating clients try to inject corruption by flipping the labels of images from $0$ to $1$ and vice-versa. In this case, the server needs to employ an appropriate attacker detection algorithm to identify these colluding attackers during the aggregation process. Once detected, their models may be excluded from the aggregation to neutralize the impact of the attackers.

\begin{figure}[t]
    \centering
    \includegraphics[scale=.32]{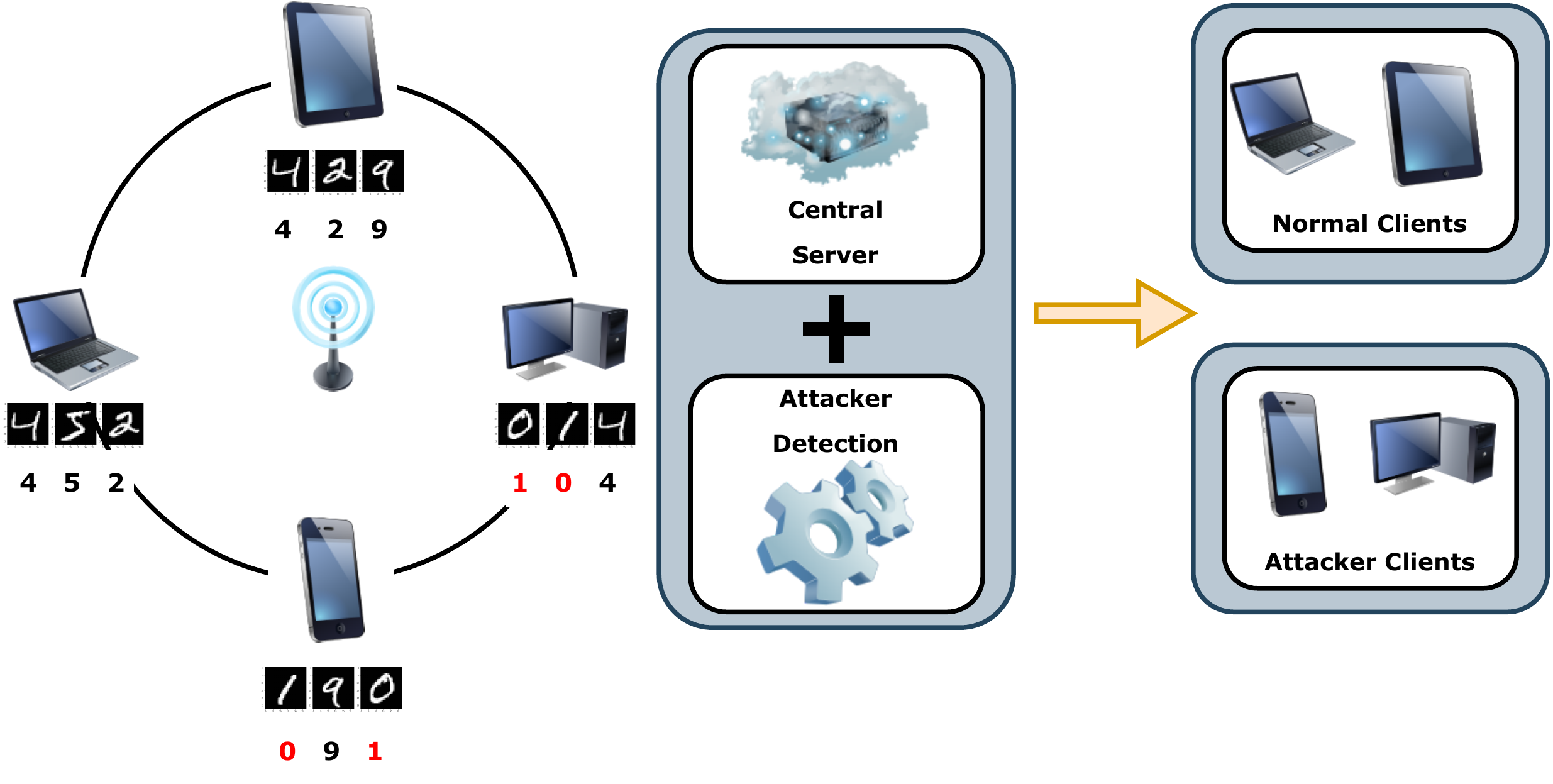}   
    \vspace{-0.05in}
    \caption{Illustrating FL setup with label flipping attacks on the local data shards by changing label $0$ to $1$ and vice-versa. The attacker detection enables server to distinguish normal client from attackers.}
    \label{fig:char}
    \vspace{-0.25in}
\end{figure}
    
In recent years, FL has spurred an active research stream to defend against possible attacks, which can be broadly categorized into two types: (i) untargeted -- adversary aims to influence the convergence of global model by corrupting the whole local model~\cite{blanchard2017machine,wu2020federated,xie2019zeno}; some attacks in this category include random noise addition to local model/gradients, sign flipping, or model replacement, and (ii) targeted -- adversary attempts to misclassify specific set of samples (mostly belong to one particular class) while minimally affecting the model performance on other classes~\cite{259745,steinhardt2017certified,bhagoji2019analyzing, bagdasaryan2020backdoor, xie2021crfl}; in this category, the common attacks are label flipping and backdoor. Further, some researchers have focused on robust aggregation methods such as Krum and Multi-Krum~\cite{blanchard2017machine}, median and trimmed median~\cite{yin2018byzantine}, and GoeMed ~\cite{chen2017distributed}.

\noindent $\bullet$ \textbf{Motivation:} Our work is motivated by the following limitations of existing works. (i) The robust aggregation methods~\cite{blanchard2017machine,yin2018byzantine, chen2017distributed,wu2020towards} mostly extended the stochastic gradient descent (SGD) to aggregate the local models while assuming independent and identically distributed (IID) data across the clients, however, in the FL, the heterogeneous nature of devices produces non-IID data. Moreover, these methods can only minimize the adverse effect on the global model but do not mitigate the effect fully. On the flip side, our work primarily focuses on \textit{complete mitigation of the attackers' impact} by excluding their models from aggregation. (ii) The existing methods that detect targeted attacks can work accurately only when the number of attackers is less than the number of normal clients. Though FoolsGold~\cite{259745} has overcome this limitation, it requires many FL rounds, thus delaying the detection process, and meanwhile, the attackers keep injecting the corruption. 

In this work, we address the problem: {\em how to secure FL against collusive attackers posing label flipping attacks?} To solve this, we propose two graph theoretic algorithms exploiting Maximum Spanning Tree (MST) and $k$-Densest graph problems. Particularly, we make the following contributions:
    \begin{itemize}
    \item We propose two novel attacker detection algorithms, called MST-AD and Density-AD, by leveraging the correlation computed over the gradients \footnote{The terms ``gradients" and ``weight updates" are used interchangeably.} of the clients. Since collusive attackers have a common objective, their models are highly correlated and have the potential to reveal their presence through MST and $k$-densest graph. 
        \item By incorporating MST-AD and Density-AD in the aggregation, we enable the server to identify the poisoned local models and exclude them.
        \item We experimentally evaluate the effectiveness of the proposed algorithms on two benchmark image classification datasets with evidence of their superiority over three different existing algorithms.
        \end{itemize}

The rest of the paper is organized as follows. Section~\ref{sec3} describes our FL setup along with the considered threat model. Section~\ref{sec4} proposes the attacker detection algorithms while Section~\ref{sec5} evaluates these algorithms on two benchmark datasets. Finally, the paper is concluded in Section~\ref{sec6}.

\section{Problem Description}\label{sec3}
\label{sec:problem formulation}

We consider a standard FL setup with a central server and $\mathcal{C}$ clients of which $\mathcal{M}$ clients are attackers (i.e., malicious in nature). Each client $c_i$ possesses a local training data shard $D_i = \{\mathbf{X}, \mathbf{y}\}$ where $\mathbf{X}$ denotes the set of training samples with labels $\mathbf{y}$. For a classification task, the server initializes a global model $W^t$ for round $t=1$ and dispatches it to all the clients who retrain this model on their local data. Let $\delta_i^t$ be the gradients (weights update, i.e., $W^t - w_i^t$) obtained by client $c_i$, which is sent back to the server in round $t$. To this end, the server does aggregation as 
\begin{align} \textstyle
   W^{t+1} = W^{t} + \sum_{i=1}^{\mathcal{C}-\mathcal{M}} p_i \delta_i^t + \sum_{j=1}^{\mathcal{M}} p_j \delta_j^t,  
\end{align}
where $p_i$ is the weight of client $c_i$ computed over the percentage of data samples the client possesses, and $\sum_i p_i = 1$.

\noindent \textbf{Objective:} To mitigate the effect of attackers on global model, the term $\sum_{j=1}^{\mathcal{M}}p_j\delta_j^t$ should be nullified. To achieve this, we aim to correctly identify all $\mathcal{M}$ attackers by leveraging the correlation between $\delta_i$ and $\delta_j$ for each pair of clients and $i \neq j$. 

\textit{Assumptions:} Our FL setup assumes on following:
(i) data across participants follow non-IID distribution;
(ii) no client has access to the local model of others;
(iii) client has no control over the aggregation algorithm;
(iv) the number of attackers is at least $2$ to realize collusion case.

\noindent \textbf{Threat model:} Our threat model is limited to targeted attacks done by compromised devices. Particularly, we focus on the label flipping attacks posed by the colluding devices even when they overwhelm normal clients. Attackers manually change a particular label (say `A') to another label (say `B') and vice-versa, in their local datasets before training the local model. Prior works~\cite{tolpegin2020data,259745} have also demonstrated that the colluding devices can promote the poisoning effect rapidly. 

\textit{Attacker Capabilities:} 
The adversary has full control of the compromised devices, however, can not access the local data or model of the benign devices. 

\section{Proposed algorithms}\label{sec4}
In this work, we propose two attacker detection algorithms, MST-AD and Density-AD, essentially named after the underlying graph theory concepts, by leveraging the correlation between the clients' gradients.  
Since the attackers train on a poisoned dataset (with flipped labels), they should stay closer to each other, i.e., their gradients should show higher similarity in some spaces. 
Through extensive experiments, we found that the correlation between the clients' gradients can effectively separate collusive attackers from normal clients.     
We define the correlation between any two clients $i$ and $j$ as
\begin{equation}\textstyle
    r_{ij} = \frac{\sum_{d} (\delta_{i} - \bar{\delta}) (\delta_{j} - \bar{\delta})}{\sqrt{\sum_{d} (\delta_{i} - \bar{\delta})^2 \  \sum_{d}(\delta_{j} - \bar{\delta})^2}}, \label{eq:corr}
\end{equation}
where $\bar{\delta}$ represents the mean gradients and $d$ is the dimension of the gradients' matrix (the length of the weight vector). 

\begin{figure}[h]
    \centering
    \includegraphics[scale=.65]{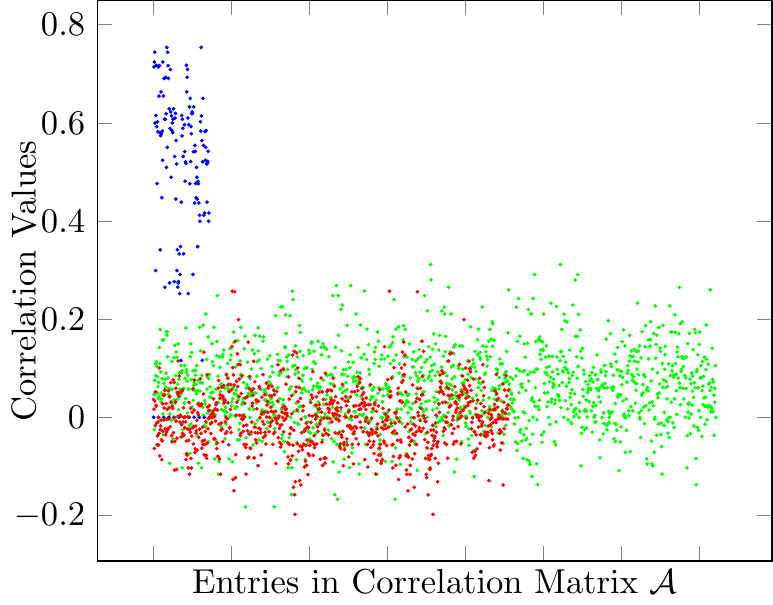}
    \caption{Correlation values between Attacker-Attacker (blue), Normal-Normal (green) and Attacker-Normal (red).}
    \label{fig:corr}
    \vspace{-0.1in}
\end{figure}

\noindent \textbf{An empirical observation:} 
By considering an FL setup with $50$ clients out of which $25\%$ are attackers, we experiment on Fashion-MNIST~\cite{xiao2017fashion} dataset with non-IID data, and the correlations between the clients are shown in Fig.~\ref{fig:corr}.
It is easy to see that the correlation between the attackers is always greater than that between two normal clients, which in turn is greater than the correlation between an attacker and a normal client. Though the above statement does not hold for every single correlation, it suffices to distinguish attackers from normal clients using correlation values. 

To this end, to design our algorithms we make the following \textit{assumption about the correlation} -- given a set of clients $\mathcal{C}$ and a set of attackers $\mathcal{M} \subset \mathcal{C}$, we have the inequality
\begin{align}
    r_{ip} < r_{ij} < r_{pq},
    \label{eq:inequality correlations}
\end{align}
where $p,q \in \mathcal{M}, p \neq q$ and $i,j \in \mathcal{C}\setminus\mathcal{M}, i \neq j$.

By using Eq.~\eqref{eq:corr}, we define a correlation matrix $\mathcal{A} \in \mathbb{R}^{n \times n}$ where an entry $\mathcal{A}_{ij}$ corresponds to the correlation coefficient $r_{ij}$ between the clients $i$ and $j$ with $r_{ii} = 0$.
This lets us create a graph with $n$ vertices corresponding to the $n$ clients participating in FL and the edge weight between a pair of clients $i$ and $j$ corresponds to the entry $\mathcal{A}_{ij}$. The symmetric nature of the matrix makes the graph a complete undirected graph.

Fig.~\ref{figgraph} shows a representative graph for an FL Setup with $6$ clients labelled $\{c_1,c_2,c_3,c_4,c_5,c_6\} \in \mathcal{C}$. 
The clients $c_2$, $c_5$, and $c_6$ are collusive attackers poisoning the model and thus the set of attackers $\mathcal{M} = \{c_2,c_5,c_6\}$. 
The remaining clients belong to the set of normal clients and given as $\mathcal{C}\setminus\mathcal{M} = \{c_1,c_3,c_4\}$. The gradients with clients are provided for representation purposes only and the correlation values between these gradients are shown as edge weights.
As the clients $c_2$, $c_5$, and $c_6$ are attackers, the edge between these clients has a higher value as compared to the other edges in the graph. Similarly, the edges between the attacker and normal client (e.g., between $c_2$ and $c_1$) have a lower value as compared to the edges between two normal clients. 

\begin{figure}[ht]
	\vspace{-0.12in}
	\centering
	\includegraphics[scale=0.4]{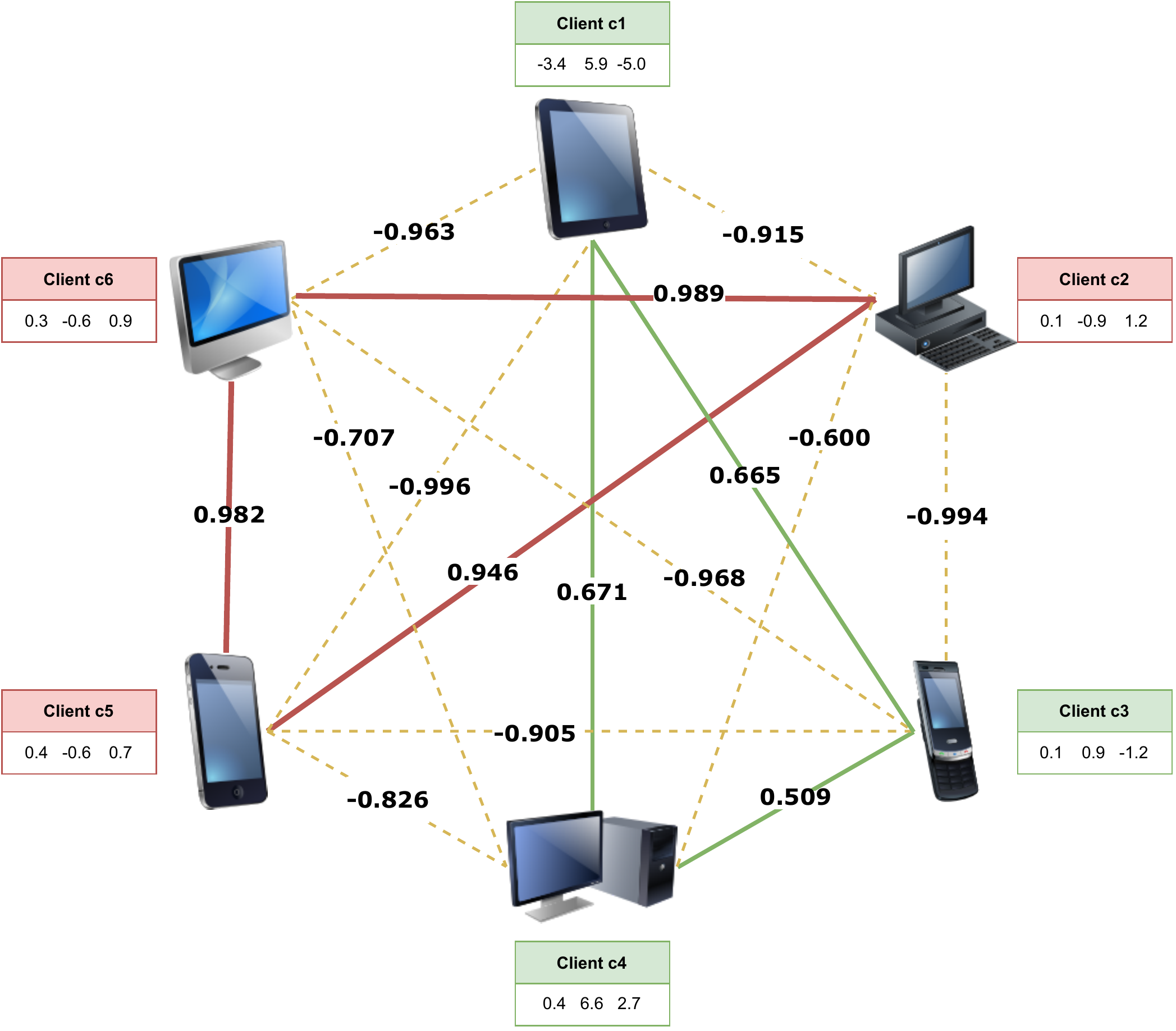}
	\caption{Example graph showing the correlations (weight on edges) between clients. Attacker-Attacker edges (red) with a higher weight, Attacker-Normal edges (dashed yellow) with a lower weight, and Normal-Normal edge (green). Clients $c_2, c_5$, and $c_6$ are the attackers. Numeric values with each client are representing the gradients.
	}
	\label{figgraph}
	\vspace{-0.15in}
\end{figure}

\subsection{MST-AD Algorithm}\label{seca}
In this section, we exploit the graph realization obtained from the correlation matrix to create an MST, similar to our previous work~\cite{ranjan2022leveraging}, which we leverage to distinguish attackers from normal clients. 
We recall that an MST is a spanning tree of a weighted graph having maximum weight, i.e., on a set of $n$ clients, the tree is composed of $n-1$ edges of maximum weight, subject to a standard constraint that the selected edges do not form a cycle.

Specifically, from the set of edges in the graph, an edge $edge$ is chosen and added to the tree $trees$ if $edge$ has the maximum weight among all the remaining edges in the graph and $edge$ does not form a cycle on the edges of $trees$. 
Following this, the $edge$ is discarded from the graph, and the edge with the next highest weight is chosen and the process continues till the MST is created. Upon the creation of the corresponding MST, the edge with the lowest weight is chosen and discarded, which results in two sub-trees; the one with the higher average edge weight corresponds to attackers whose gradients are later excluded from the aggregation to mitigate their impact. 

Algorithm~\ref{alg:3} illustrates the above procedure for creating the sub-trees (essentially MST) over the set of clients. The algorithm starts by initializing the list of trees as an empty set and sorting the edges of the graph in non-decreasing order of their weights (Line~{3}).
We pick the first $n-1$ edges from the sorted set and add them to the tree (Lines~{4--7}).
Considering that our assumption (inequality defined in Eq.~\ref{eq:inequality correlations}) holds, the edges connecting the attackers should be included in the MST. Moreover, there would exist a single Attacker-Normal edge having the lowest weight among all the edges in formed MST. The deletion of such edge results in two sub-trees. 
Since the edges with higher weights exist between the attackers, they would form a single connected sub-tree.
\SetAlFnt{\small}
\begin{algorithm}[ht]
\caption{MST-AD Algorithm}{\label{alg:3}}
\textbf{Input:} Correlation matrix $ \mathcal{A}$ \\
\textbf{Output:} Set of attackers ($Atk$)\\ 
$trees\! \gets\! \emptyset$; $i\!\!\!=\! 0$; $\mathcal E\! \gets\!$ sorted edges in non-increasing order\label{s5}\\
\While{$|trees| < n-1$ \label{s6}}{
    \If{$\mathcal{E}[i]$ { does not form cycle in} $trees$ \label{s7}}{
        $trees \gets trees \cup \mathcal{E}[i] \label{s8}$
    }
    $i$++
}
$subT_1, subT_2 \gets$ Remove lowest weighted edge from $trees$\\
\nonl /* Let avg\_weight($\cdot$) computes average weight of tree */\\
\If{avg\_weight($subT_1$) $>$ avg\_weight($subT_2$) }{
$Atk  \gets subT_1 $ \label{s11}
}
\Else{$Atk \gets subT_2 $}

\textbf{return} $Atk$ 
\end{algorithm}
\vspace{-0.1in}
\begin{theorem}
{\em Given a correlation matrix $\mathcal{A}$, let the inequality (Eq.~\ref{eq:inequality correlations}) 
hold for any pair of clients, then Algorithm~\ref{alg:3} returns the complete and correct set of collusive attackers.} 
\end{theorem}
\begin{proof}
Given the inequality $r_{ip} < r_{ij} < r_{pq}$ for any pair of attackers $p,q \in \mathcal{M}$, and any pair of normal clients $i,j \in \mathcal{C}\setminus\mathcal{M}$.
Let $\mathcal{E}$ be the set of ordered edges of the graph induced by $\mathcal{A}$.
Indeed, $\mathcal{E}$ is divided into three contiguous subgroups, the group of edges between any pair of attackers, followed by the group of edges between any pair of normal clients, followed by the last group formed by the edges between normal and attacker.
Then the proof follows directly from the construction of the MST.
In fact, by definition of MST, we have to select $n-1$ edges from $\mathcal{E}$ starting from the edges with maximum weights.
Following the MST construction, we will have a sub-tree composed of all the malicious clients (edges with higher weight), one sub-tree composed of all the normal clients (the second group of edges in $\mathcal E$), and, finally, one single edge (the one with lower weight in such a tree) between the two sub-trees.
Thus, by removing the edge with the lowest weight, we are able to distinguish between normal and attacker clients.
\end{proof}

\subsection{Density-AD Algorithm}\label{secb}
This section introduces another detection algorithm by leveraging the concept of $k$-densest graph which essentially is a maximum density sub-graph with exactly $k$ vertices. 
Assuming the inequality defined in Eq.~\ref{eq:inequality correlations} holds, our problem can be realized as a $k$-densest sub-graph problem in which the objective is to find the $k$ vertices with the highest average weighted degree.
Note that the value of $k$ in our problem (i.e., the number of attackers) is not known in advance. 
Instead, we aim to find out the $k$ vertices whose removal maximizes the density of the remaining sub-graph. The density of a graph is defined as the average of all weighted degrees of the vertices. Given the correlation matrix $\mathcal{A}$, the graph density can be formally defined as 
\begin{align}
   density(\mathcal{A}) = \frac{2\sum_{i=1}^n \sum_{j=1}^n r_{ij}}{n(n-1)}.
\end{align}

\begin{definition}[\textbf{Sparse vertex}]
{\em A vertex $v$ of a graph $\mathcal{G}$ is sparse if its removal increases the density of the graph $\mathcal{G}-v$.}  
\end{definition}

\medskip
Our algorithm iterates over each vertex of the graph to identify whether it is a {\em sparse} vertex or not. If a vertex is sparse, it is removed permanently from the graph otherwise it is replaced back in that graph. Once all the vertices are traversed through successive iterations, the remaining sub-graph with $k$ vertices corresponds to the potential attackers because they have the highest correlations among themselves.

\SetAlFnt{\small}
\begin{algorithm}[ht]
\caption{Density-AD Algorithm}{\label{alg:kd}}
\textbf{Input:} Correlation matrix $ \mathcal{A}$ \\
\textbf{Output:} Set of attackers ($Atk$)\\ 
$sparse\_list \gets \emptyset$ \\

\While{$i = n \text{ down to } 1$ }{ \label{r4}

        $\mathcal{B} \gets \mathcal{A} \setminus \mathcal{A}[i]$ \label{r5}\\
        \If{$density(\mathcal{B}) > density(\mathcal{A})$ \label{r6}}{$sparse\_list \gets sparse\_list\cup \mathcal{A}[i]$\label{r7}\\
        $\mathcal{A} \gets \mathcal{A} \setminus \mathcal{A}[i]$} \label{r8}

}
\If{$density(sparse\_list) > density(\mathcal{A})$}{$Atk \gets sparse\_list$ \label{r9}}
\Else
{$Atk \gets \mathcal{A}$ \label{r10}}

\textbf{return} $Atk$ 
\end{algorithm}

The overall steps of the proposed Density-AD algorithm are reported in Algorithm~\ref{alg:kd}.
First the list $sparse\_list$ is created for storing the attackers detected during each iteration.
The loop at Line~\ref{r4} iterates $n$ times to testify the sparse nature of each vertex.   
The set $\mathcal{B}$ is temporarily used to store the elements of set $\mathcal{A}$ (Line~\ref{r5}) excluding the $i^{th}$ vertex. 
If the $i^{th}$ vertex is a sparse vertex in the graph, the corresponding density of $\mathcal{B}$ will be higher than the density of $\mathcal{A}$ (Line~\ref{r6}) and $i^{th}$ vertex is then appended to the list $sparse$ and subsequently removed from the set $\mathcal{A}$ (Line~\ref{r8}). 
Next, the density of the nodes in the $sparse$ list is compared with the density of the remaining nodes in $\mathcal{A}$ (Line~\ref{r9}-~\ref{r10}), and the set with a higher density is marked as the set containing the colluding attackers.

\section{Experimental Evaluation}\label{sec5}
In this section, we evaluate the effectiveness of our proposed attacker detection algorithms and analyze the obtained results with a critical comparison with popular existing algorithms.

\subsection{Experimental Setup}
We consider the task of image classification using deep neural networks consisting of $2$ Convolutional Neural Network~\cite{lecun1989backpropagation} layers followed by $3$ fully-connected layers. 
We use two benchmark datasets:  MNIST~\cite{lecun1998gradient} and Fashion-MNIST (FMNIST in short)~\cite{xiao2017fashion}, each comprising of $60000$ training and $10000$ testing greyscale images divided equally in $10$ classes. 
For each dataset, the samples were randomly partitioned into $n = 50$ disjoint subsets and each of that is assigned to a single client. Inspired by~\cite{bagdasaryan2020backdoor}, we adopt Dirichlet distribution with parameter $\alpha = 0.9$, for the partitioning. We simulate the attacker as follows: for the MNIST dataset, the labels of all images with `0' and `1' are flipped and for the FMNIST dataset, labels of all images of ``T-Shirt" and ``Trouser" are flipped. 
Note that the adopted label flipping is bi-directional. Further, inspired by~\cite{259745}, we abbreviate the attack scenarios as $A{-}m$ attacks where $m$ is the percentage of attackers to the total number of clients and $A$ is the shorthand for the term `attack'.
For instance, an $A\!-\!5$ attack would refer to the scenario with $5\%$ of the total clients as collusive attackers.

\subsection{Performance Metrics}

The evaluation is carried out on the test data while comparing our algorithms with competitive detection algorithms. We consider FoolsGold~\cite{259745} and GeoMed~\cite{yin2018byzantine} algorithms for comparison. Besides, our experiments also included federated averaging (FedAvg)~\cite{mcmahan2017communication} as the baseline. We employ the following metrics to quantify the performance:
 (i) \textbf{Test accuracy}, the proportion of correctly classified samples in the test set;
  (ii) \textbf{Attack Success Rate (ASR)}, the proportion of the targeted samples incorrectly classified in the test set. In the context of label flipping, the value corresponds to the ratio of the number of misclassified flipped labels to the total number of labels flipped by  adversaries~\cite{259745}. (iii) \textbf{ED}, the earliest round at which all the attackers got detected correctly.

\subsection{Results}
While reporting the experimental results in this section, we mainly attempt to answer the following questions: (i) How does the training loss decrease over 30 FL rounds? (ii) What is the impact of colluding attackers on the test accuracy and F1 Score of all the algorithms? (iii) How efficiently and early do the proposed algorithms detect the label-flipping attackers? 

\subsubsection{Training loss over FL rounds} \label{1}
Figs.~\ref{l1} and~\ref{l2} report the obtained loss on the training data over the communication rounds on both the datasets for $A{-}5$ and $A{-}70$ attacks, respectively. 
The effects of collusion can be seen with the algorithms employing central measures (FedAvg and GeoMed) showing higher loss as compared to MST-AD and Density-AD. 
While the losses for all the algorithms converge in $A{-}5$ attack, it diverges a lot for the existing algorithms when the attackers overwhelm the normal clients, i.e., in the case of $A{-}70$ attack. 
As the number of attackers increases, the existing algorithms could not detect and eliminate the effect of the attackers, causing a large loss as illustrated in Fig.~\ref{l2}. It is interesting to notice that the performance of the existing attacker detection algorithms, i.e., FoolsGold and GeoMed, turn out to be worse than the baseline FedAvg algorithm. This is mainly caused due to incorrect elimination of the normal clients (detected wrongly as attackers) from the aggregation thus indirectly strengthening the collusion attack. 
This is evident in the case of the FoolsGold algorithm for the MNIST dataset where a steep jump in the loss appears after initial rounds of training. 

In Fig.~\ref{l2}, the proposed algorithms show comparable loss to the existing algorithms for the initial rounds, but the loss decreases sharply afterward, which can indeed be verified by the earliest round of detection (ED) reported in Table~\ref{tab1}.

  \begin{figure}[h]
  \vspace{-0.1in}
    \centering
    \minipage{0.235\textwidth}
    \includegraphics[width=1\linewidth]{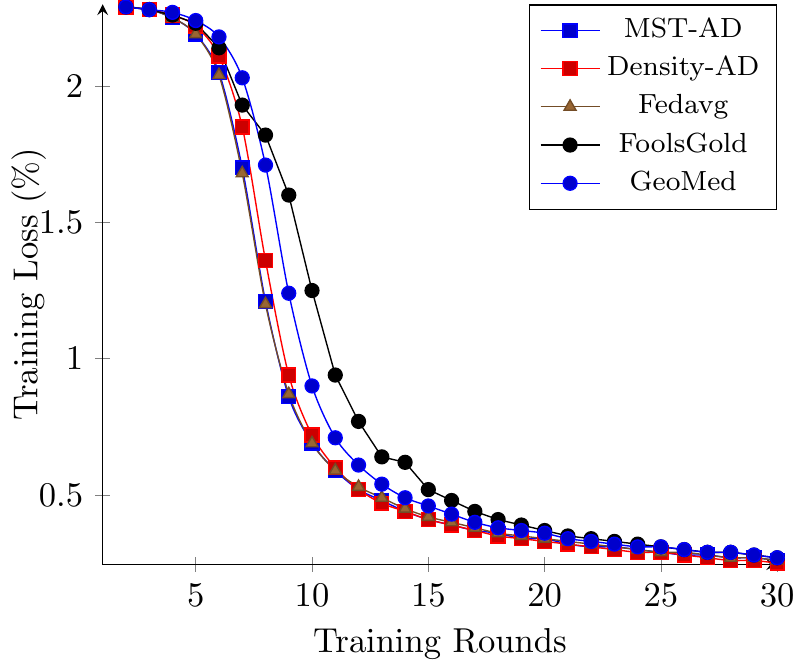}
    \subcaption{\footnotesize{Training Loss for MNIST}}
\endminipage\hfill
\minipage{0.235\textwidth}
\centering
    \includegraphics[width=1\linewidth]{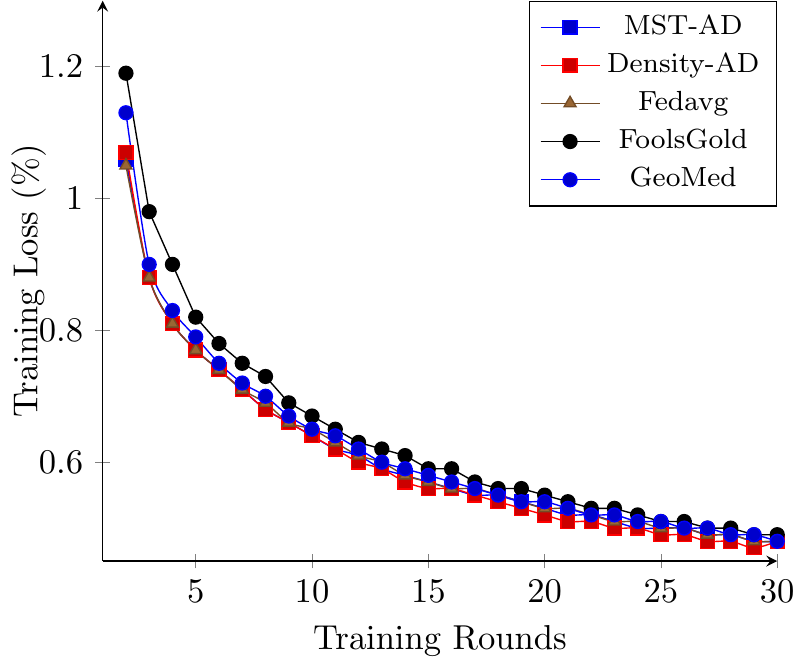}
    \subcaption{\footnotesize{Training Loss for FMNIST}}
\endminipage\hfill 
\caption{Training loss over FL rounds for $A{-}5$ attack.}
    \label{l1}
    \vspace{-0.1in}
\end{figure}

  \begin{figure}[h]
    \vspace{-0.2in}
    \centering
    \minipage{0.235\textwidth}
    \includegraphics[width=1\linewidth]{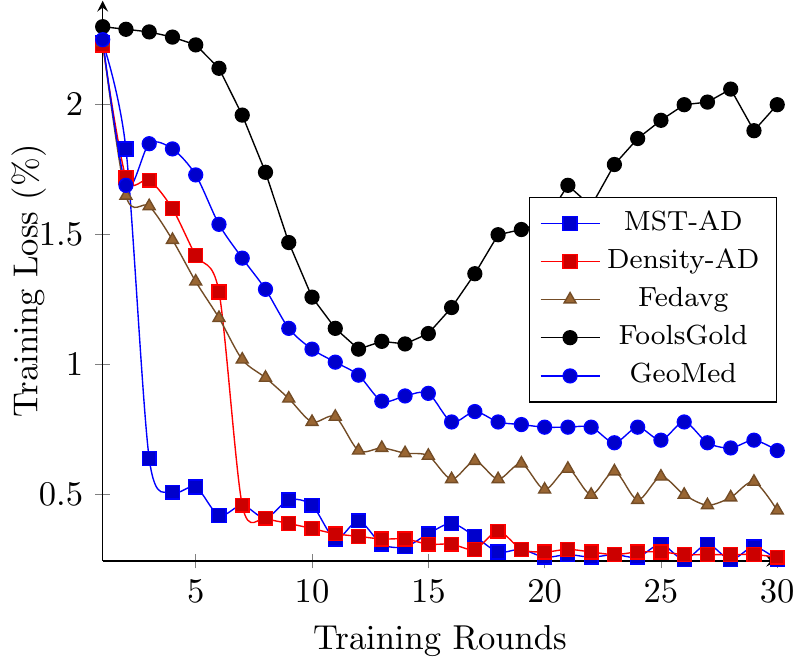}
    \subcaption{\footnotesize{Training Loss for MNIST}}
\endminipage\hfill
\minipage{0.235\textwidth}
\centering
    \includegraphics[width=1\linewidth]{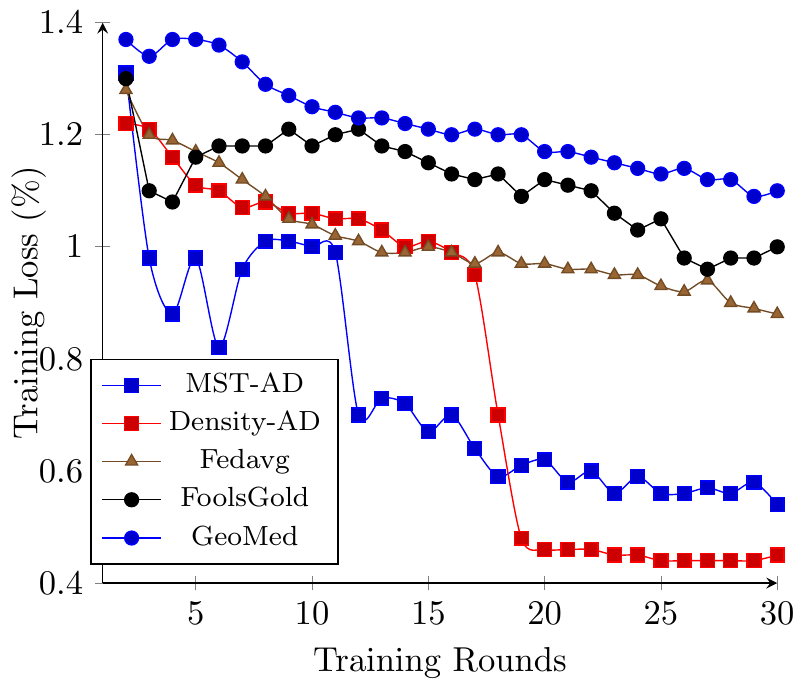}
    \subcaption{\footnotesize{Training Loss for FMNIST}}
\endminipage\hfill     
    \caption{Training loss over FL rounds for $A{-}70$ attack.}
    \label{l2}
    \vspace{-0.1in}
\end{figure}

\subsubsection{Impact of colluding attackers on test accuracy}\label{2}
Next, we report the results on test data, for both datasets, with a varying number of attackers in Fig.~\ref{fig:accuracyf1}.
It is clear that the proposed algorithms can maintain consistent performance with a larger number of attackers, however, the existing algorithms employing central measures like mean and median suffer from performance loss especially when the attackers overwhelm the normal clients.
This can be attributed to the compared algorithms incorrectly classifying normal clients as attackers and excluding them from the aggregation process, thereby degrading the classification accuracy.

Similar observations can be made from the F1 Score for the algorithms reported in parts (b) and (d) of Fig.~\ref{fig:accuracyf1}. Since the proposed algorithms are able to detect the full set of attackers, the obtained F1 Score values do not show much drop even when the number of attackers is more than $50\%$.

  \begin{figure}[h]
    \centering
    \minipage{0.235\textwidth}
    \includegraphics[width=1\linewidth]{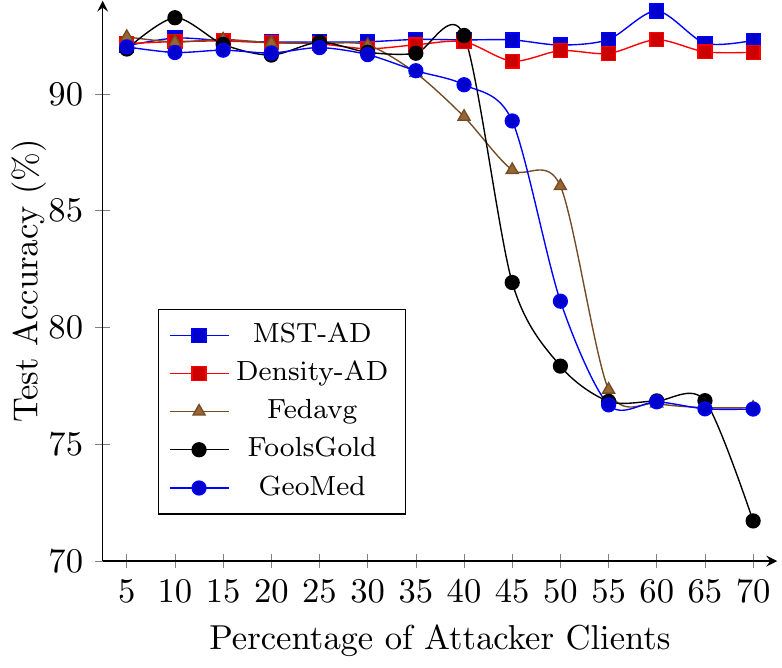}
    \subcaption{\footnotesize{Test Accuracy for MNIST}}
\endminipage\hfill
\minipage{0.235\textwidth}
\centering
    \includegraphics[width=1\linewidth]{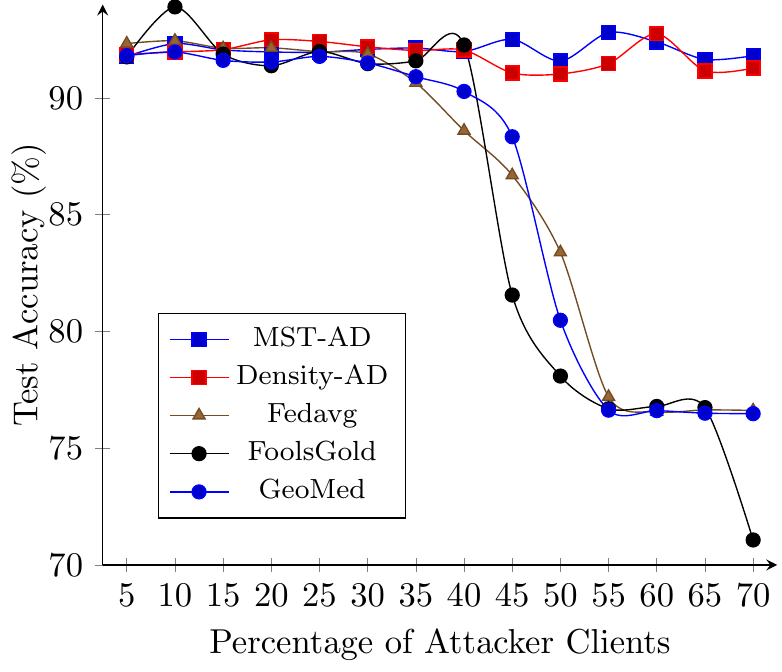}
    \subcaption{\footnotesize{Test F1 Score for MNIST}}
\endminipage\hfill  
    \centering
    \minipage{0.235\textwidth}
    \includegraphics[width=1\linewidth]{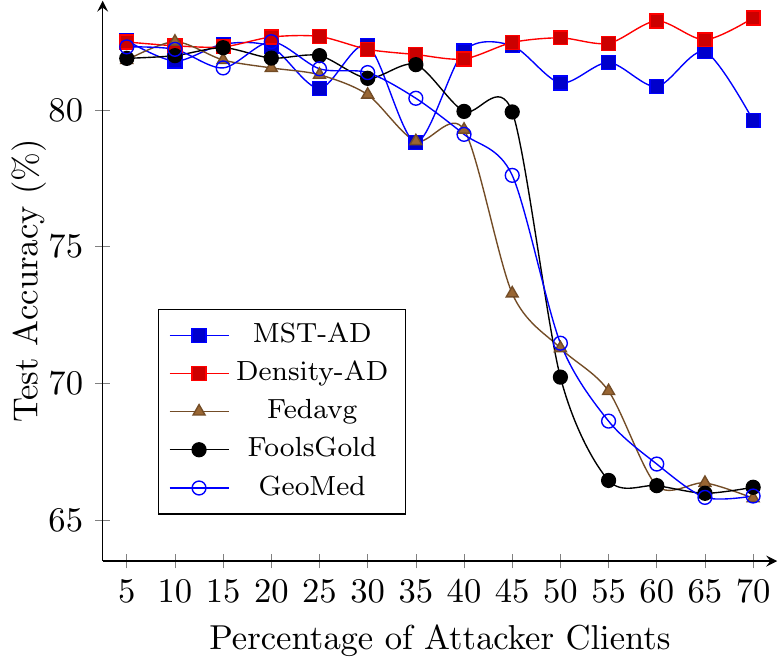}
    \subcaption{\footnotesize{Test Accuracy for FMNIST}}
\endminipage\hfill
\minipage{0.235\textwidth}
\centering
    \includegraphics[width=1\linewidth]{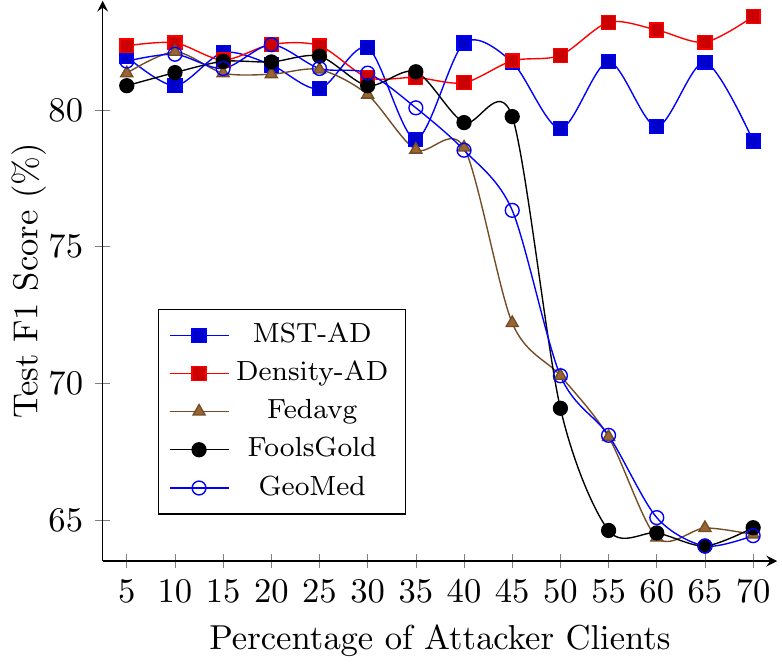}
    \subcaption{\footnotesize{Test F1 Score for FMNIST}}
\endminipage\hfill  
\caption{Accuracy and F1 Score results under varying attack scenarios.}
    \label{fig:accuracyf1}
\end{figure}

\subsubsection{Analyzing ASR}\label{analyzing_asr}
The ASR and the earliest detection round (ED) of all the attackers are presented in Tables~\ref{tab2} and~\ref{tab1} for MNIST and FMNIST datasets, respectively. It is easy to observe that the proposed algorithms are able to maintain a lower ASR even when the proportion of attackers rises.
As the proposed algorithms are able to eliminate the effects of the colluding workers, the successful number of attacks remains lower while the number of samples targeted by the attackers' increases, thus yielding a decrease in ASR. This is in contrast with the existing algorithms yielding a higher ASR following an increase in the number of colluding attackers.
Among the considered existing algorithms, FoolsGold can successfully detect all the attackers only when their percentage is low however it fails in the majority of the cases and thus most entries in ED are marked by $*$.
The proposed algorithms, on the other hand, are able to consistently identify the full set of attackers which is also reflected by their ASR. 

\begin{table}[htb]
\caption{ASR and ED with varying attack scenarios for MNIST. We do not report ED for FedAvg and GeoMed algorithms as they do not focus on detection. [\textbf{ASR} -- lower is better. \textbf{ED} -- $*$ means the algorithm could not detect full set of attackers up to $30$ rounds]}
\label{tab2}
\centering
\resizebox{\columnwidth}{!}{%
\begin{tabular}{|c|cc|cc|cc|c|c|}
\hline
\textbf{Atk} & \multicolumn{2}{c|}{\textbf{MST-AD}} & \multicolumn{2}{c|}{\textbf{Density-AD}}                & \multicolumn{2}{c|}{\textbf{FoolsGold}}           & \textbf{FedAvg} & \textbf{GeoMed} \\ \hline
                 & \multicolumn{1}{c|}{\textbf{ASR}} & \textbf{ED} & \multicolumn{1}{c|}{\textbf{ASR}} & \textbf{ED} & \multicolumn{1}{c|}{\textbf{ASR}} & \textbf{ED} & \textbf{ASR}   & \textbf{ASR} \\ \hline
$\mathbf{A{-}10}$    & \multicolumn{1}{c|}{0\%}          & 11    & \multicolumn{1}{c|}{0\%} &    10     & \multicolumn{1}{c|}{0\%}       & $*$             & 0\%      & 0\%     \\ \hline
$\mathbf{A{-}15}$    & \multicolumn{1}{c|}{0\%}          & 9      & \multicolumn{1}{c|}{0\%} &     20       & \multicolumn{1}{c|}{0\%}       & $20$             & 0.3\%   & 0\%        \\ \hline
$\mathbf{A{-}20}$     & \multicolumn{1}{c|}{0\%}        & 9      & \multicolumn{1}{c|}{0\%} &       15     & \multicolumn{1}{c|}{0\%}       & $18$            & 0.25\%      & 0\%       \\ \hline
$\mathbf{A{-}25}$     & \multicolumn{1}{c|}{0\%}          & 10       & \multicolumn{1}{c|}{0\%} &      19     & \multicolumn{1}{c|}{0\%}       & $17$             & 3\%     & 0.2\%      \\ \hline
$\mathbf{A{-}30}$     & \multicolumn{1}{c|}{0\%}          & 11       & \multicolumn{1}{c|}{0\%} &     14     & \multicolumn{1}{c|}{0\%}       & $25$             & 1.67\%     &  0\%      \\ \hline
$\mathbf{A{-}35}$     & \multicolumn{1}{c|}{0\%}          & 9       & \multicolumn{1}{c|}{0\%} &      15     & \multicolumn{1}{c|}{0\%}       & $24$             & 13\%    & 0.86\%       \\ \hline
$\mathbf{A{-}40}$     & \multicolumn{1}{c|}{0\%}        & 10      & \multicolumn{1}{c|}{0\%} &       16    & \multicolumn{1}{c|}{0\%}          & $*$             & 34\%    & 1.5\%       \\ \hline
$\mathbf{A{-}45}$     & \multicolumn{1}{c|}{0\%}          & 11     & \multicolumn{1}{c|}{0\%} &       17      & \multicolumn{1}{c|}{2.5\%}       & $*$             & 19.2\%    & 28.3\%       \\ \hline
$\mathbf{A{-}50}$     & \multicolumn{1}{c|}{0\%}          & 3     & \multicolumn{1}{c|}{0\%} &       18      & \multicolumn{1}{c|}{3.4\%}       & $*$             & 25.1\%    & 25\%       \\ \hline
$\mathbf{A{-}55}$    & \multicolumn{1}{c|}{0\%}          & 4     & \multicolumn{1}{c|}{0\%} &       20      & \multicolumn{1}{c|}{36.7\%}       & $*$             & 60.18\%    & 42.7\%       \\ \hline
$\mathbf{A{-}60}$     & \multicolumn{1}{c|}{0\%}          & 18     & \multicolumn{1}{c|}{0\%} &       15      & \multicolumn{1}{c|}{77.4\%}       & $*$             & 100\%    & 74.41\%       \\ \hline
$\mathbf{A{-}65}$     & \multicolumn{1}{c|}{0\%}          & 11     & \multicolumn{1}{c|}{0\%} &       18      & \multicolumn{1}{c|}{88\%}       & $*$             & 99.7\%    & 78.5\%       \\ \hline
$\mathbf{A{-}70}$     & \multicolumn{1}{c|}{0\%}          & 11     & \multicolumn{1}{c|}{0\%} &       9      & \multicolumn{1}{c|}{100\%}       & $*$             & 96.39\%    & 94.9\%       \\ \hline
\end{tabular}
}
\end{table}

\begin{table}[htb]
\caption{ASR and ED with varying attack scenarios for FMNIST. We do not report ED for FedAvg and GeoMed algorithms as they do not focus on detection. [\textbf{ASR} -- lower is better. \textbf{ED} -- $*$ means the algorithm could not detect full set of attackers upto $30$ rounds]}
\label{tab1}
\centering
\resizebox{\columnwidth}{!}{%
\begin{tabular}{|c|cc|cc|cc|c|c|}
\hline
\textbf{Atk} & \multicolumn{2}{c|}{\textbf{MST-AD}} & \multicolumn{2}{c|}{\textbf{Density-AD}}                & \multicolumn{2}{c|}{\textbf{FoolsGold}}           & \textbf{FedAvg} & \textbf{GeoMed} \\ \hline
                 & \multicolumn{1}{c|}{\textbf{ASR}} & \textbf{ED} & \multicolumn{1}{c|}{\textbf{ASR}} & \textbf{ED} & \multicolumn{1}{c|}{\textbf{ASR}} & \textbf{ED} & \textbf{ASR}   & \textbf{ASR} \\ \hline
$\mathbf{A{-}10}$     & \multicolumn{1}{c|}{3\%}          & 3    & \multicolumn{1}{c|}{3.5\%} &    23     & \multicolumn{1}{c|}{4.52\%}       & $*$             & 2.5\%      & 2.5\%     \\ \hline
$\mathbf{A{-}15}$     & \multicolumn{1}{c|}{2.67\%}          & 7      & \multicolumn{1}{c|}{2\%} &     14       & \multicolumn{1}{c|}{1.64\%}       & $*$             & 3\%   & 2\%        \\ \hline
$\mathbf{A{-}20}$     & \multicolumn{1}{c|}{1.5\%}        & 3      & \multicolumn{1}{c|}{1\%} &       24     & \multicolumn{1}{c|}{1.52\%}       & 22            & 3\%      & 3.25\%       \\ \hline
$\mathbf{A{-}25}$     & \multicolumn{1}{c|}{1.8\%}          & 29       & \multicolumn{1}{c|}{1.4\%} &      16     & \multicolumn{1}{c|}{1.4\%}       & $*$             & 6.6\%     & 1.4\%      \\ \hline
$\mathbf{A{-}30}$     & \multicolumn{1}{c|}{1\%}          & 18       & \multicolumn{1}{c|}{1.67\%} &     17     & \multicolumn{1}{c|}{0.84\%}       & $*$             & 2.3\%     & 2.3 \%      \\ \hline
$\mathbf{A{-}35}$     & \multicolumn{1}{c|}{0.85\%}          & 3       & \multicolumn{1}{c|}{1.08\%} &      21     & \multicolumn{1}{c|}{1.16\%}       & $*$             & 19.8\%    & 1.14\%       \\ \hline
$\mathbf{A{-}40}$     & \multicolumn{1}{c|}{0.5\%}        & 12      & \multicolumn{1}{c|}{1.63\%} &       22    & \multicolumn{1}{c|}{6\%}          & $*$             & 8.5\%    & 4.25\%       \\ \hline
$\mathbf{A{-}45}$     & \multicolumn{1}{c|}{0.67\%}          & 3     & \multicolumn{1}{c|}{0.88\%} &       20      & \multicolumn{1}{c|}{3.78\%}       & $*$             & 82.32\%    & 22.44\%       \\ \hline
$\mathbf{A{-}50}$     & \multicolumn{1}{c|}{1.1\%}          & 3     & \multicolumn{1}{c|}{0.81\%} &       21      & \multicolumn{1}{c|}{43.2\%}       & $*$             & 82.5\%    & 47.45\%       \\ \hline
$\mathbf{A{-}55}$     & \multicolumn{1}{c|}{0.81\%}          & 2     & \multicolumn{1}{c|}{0.6\%} &       15      & \multicolumn{1}{c|}{71.1\%}       & $*$             & 71.27\%    & 51.9\%       \\ \hline
$\mathbf{A{-}60}$     & \multicolumn{1}{c|}{1.83\%}          & 28     & \multicolumn{1}{c|}{0.54\%} &       16      & \multicolumn{1}{c|}{66.15\%}       & $*$             & 67.9\%    & 62.5\%       \\ \hline
$\mathbf{A{-}65}$    & \multicolumn{1}{c|}{0.77\%}          & 2     & \multicolumn{1}{c|}{0.58\%} &       19      & \multicolumn{1}{c|}{70.5\%}       & $*$             & 70.15\%    & 74\%       \\ \hline
$\mathbf{A{-}70}$    & \multicolumn{1}{c|}{1\%}          & 11     & \multicolumn{1}{c|}{0.57\%} &       18      & \multicolumn{1}{c|}{76.3\%}       & $*$             & 71.65\%    & 76.25\% \\ \hline
\end{tabular}
}
\end{table}

\subsubsection{Confusion matrices} Finally, we report the confusion matrices for FoolsGold and Density-AD in case of overwhelming attackers (i.e., for $A{-}70$) in Fig.~\ref{fig:conf}. We can clearly see that FoolsGold could not correctly classify the images of flipped labels (i.e., $0$ and $1$), whereas the proposed algorithms do not show any such confusion between these labels.  

  \begin{figure}[h!]
    \centering
    \centering
    \minipage{0.235\textwidth}
    \includegraphics[width=1\linewidth]{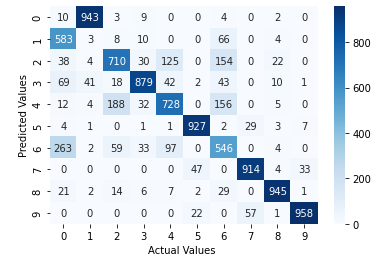}
    \subcaption{\footnotesize{FoolsGold}}
\endminipage\hfill
\minipage{0.235\textwidth}
\centering
    \includegraphics[width=1\linewidth]{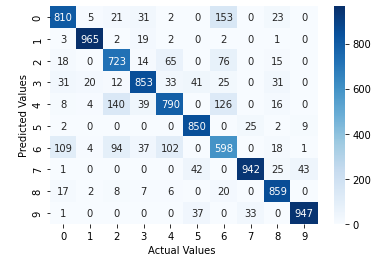}
    \subcaption{\footnotesize{Density-AD}}
\endminipage\hfill  
    \caption{Confusion matrices for FoolsGold and Density-AD algorithms in case of $A{-}70$ on FMNIST dataset.}
    \label{fig:conf}
    \vspace{-0.1in}
\end{figure}

\section{Conclusion}\label{sec6}
In this paper, we attempted to address a critical problem of FL framework which is the presence of colluding attackers. Since the attackers can harm the global model severely, their detection is of utmost need for real deployment of the FL. We proposed two graph-based algorithms, MST-AD and Density-AD, by leveraging gradients' correlation among the clients. By performing an extensive set of experiments, we validated that the proposed algorithms can maintain a low attack success rate even when the attackers overwhelm the normal clients.  

Since the proposed algorithms rely on correlation, they may not be able to detect an adversary if the FL system does not have any other adversary to collude with, which we plan to relax in the future.
In addition, the current versions of proposed algorithms are limited to label-flipping attacks only. We plan to scale our FL setup by including more types of attacks such as byzantine, backdoor, and multi-label flipping.

\bibliographystyle{IEEEtran}
\bibliography{refer}

\end{document}